\documentclass{article}

\PassOptionsToPackage{numbers, compress}{natbib}

\usepackage{wrapfig}

\usepackage[preprint]{neurips_2025}

\usepackage[utf8]{inputenc} 
\usepackage[T1]{fontenc}    
\usepackage{hyperref}       
\usepackage{url}            
\usepackage{booktabs}       
\usepackage{amsfonts}       
\usepackage{nicefrac}       
\usepackage{microtype}      
\usepackage{xcolor}         
\usepackage{amsmath}
\usepackage{amsthm}
\usepackage{amssymb}
\usepackage{mathtools}
\usepackage{algorithm}
\usepackage{algpseudocode}
\usepackage{mdframed}
\usepackage{geometry}
\usepackage{amsmath}
\usepackage{amssymb}
\usepackage{amsthm}
\usepackage{multirow}
\usepackage{graphicx}
\usepackage{bbm}
\usepackage{tikz}
\usetikzlibrary{positioning,shapes,arrows,backgrounds}
\usepackage{amsmath}
\usepackage{booktabs}
\usepackage{amssymb}
\usetikzlibrary{positioning,shapes,arrows.meta,backgrounds}
\usepackage{float}
\usepackage{thmtools} 
\usepackage{thm-restate}
\usepackage{soul} 
\usepackage{titletoc}

\usetikzlibrary{shapes,arrows,positioning,fit,backgrounds}
\newtheorem{theorem}{Theorem}

\newtheorem{definition}{Definition}
\newtheorem{corollary}{Corollary}
\theoremstyle{definition}

\newcommand{\RR}{\mathbb{R}}

\newcommand{\uq}{V^{(q)}}

\newcommand{\equi}{equiEPNN}

\newcommand{\domain}{\mathcal{V}_{\mathrm{simple}}^K}

\title{Spectral Graph Neural Networks are Incomplete on Graphs with a Simple Spectrum}

\author{%
   Snir Hordan\\
  Faculty of Mathematics\\
  Department of Applied Mathematics\\
  Technion - Israel Institute of Technology\\
  \And
   Maya Bechler-Speicher\\
Meta\\
  \And
   Gur Lifshitz\\
   Blavatnik School of Computer Science\\
    Tel-Aviv University\\
   \And
  Nadav Dym \\
Faculty of Mathematics\\
Faculty of Computer Science\\
  Technion - Israel Institute of Technology\\
 \\
}

\begin{document}

\maketitle

\begin{abstract}

Spectral features are widely incorporated within Graph Neural Networks (GNNs) to improve their expressive power, or their ability to distinguish among non-isomorphic graphs. One popular example is the usage of graph Laplacian eigenvectors for positional encoding in MPNNs and Graph Transformers. The expressive power of such Spectrally-enhanced GNNs (SGNNs) is usually evaluated via the $k$-WL graph isomorphism test hierarchy and homomorphism counting \cite{gai2025homomorphism, zhang2024expressive}. Yet, these frameworks align poorly with the graph spectra, yielding limited insight into SGNNs' expressive power. We leverage a well-studied paradigm of classifying graphs by their largest eigenvalue multiplicity \cite{10.1145/800070.802206} to introduce an expressivity hierarchy for SGNNs. We then prove that many SGNNs are incomplete even on graphs with distinct eigenvalues. To mitigate this deficiency, we adapt rotation equivariant neural networks to the graph spectra setting to propose a method to provably improve SGNNs' expressivity on simple spectrum graphs. We empirically verify our theoretical claims via an image classification experiment on the MNIST Superpixel dataset \cite{super-pixel} and eigenvector canonicalization on graphs from ZINC \cite{Irwin2012}. 


\end{abstract}

\section{Introduction}

Graph Neural Networks (GNNs) have become a ubiquitous paradigm for learning on graph-structured data. The core principle of GNNs is to maintain a representation of each graph vertex and leverage the graph structure to iteratively refine each representation by its vertex's graph neighborhood \cite{scarselli2008graph}. To enhance the purview of the vertex's neighborhood, it is common to incorporate spectral features, such as Random Walk matrices, positional encoding, and graph distances, into the refinement operation of  GNNs \cite{dwivedi2023benchmarking,ArnaizRodriguez2022, Velingker2023, zhang2023complete}. Such GNNs, which systematically incorporate spectral features within their representation refinement procedure, or Spectrally-enhanced GNNs (SGNNs)\cite{zhang2024expressive}, have gained significant traction in the graph learning community, due to their reasonable complexity and empirical benefits  \cite{huang2023stability, zhou2024stablegloballyexpressivegraph, zhang2024expressive, gai2025homomorphism}.

Understanding the expressive power of GNNs provides researchers with a framework for comparing different models and identifying their deficiencies, often leading to improvements \cite{pmlr-v235-hordan24a, maron2019provably, morris2019weisfeiler, Hordan_Amir_Gortler_Dym_2024, xu2018powerful}. These frameworks ought to characterize which graphs the GNN can distinguish among, based on the GNNs' inner workings. For instance, the Weisfeiler-Leman (WL) test, which maintains and refines vertex representations similarly to Message Passing Neural Networks, a subclass of GNNs, completely determines which graphs these models can distinguish among \cite{xu2018powerful}.

To study the expressive power of SGNNs, recent papers \cite{zhang2024expressive, gai2025homomorphism} proposed a spectrally enhanced GNN, called Eigenspace Projection GNN (EPNN), which generalizes many popular spectral graph neural networks, and analyze its expressivity via WL tests and homomorphism counting. This comparison is valuable in comparing the expressivity of SGNNs to that of their combinatorial GNN counterparts. Yet, this analysis does not yield insight into the role of the graph spectra in the distinguishing ability of these GNNs. 

To address this gap, we propose analyzing the expressive power of SGNNs via  Spectral Graph Theory, 
and in particular via the maximal eigenvalue multiplicity of a graph. As isomorphism of graphs with bounded eigenvalue multiplicity can be determined in polynomial time with the complexity depending on the bound \cite{10.1145/800070.802206}, this notion imposes a natural hierarchical classification of graphs, and SGNNs can potentially be complete on these graph classes, making this hierarchy a viable method for assessing their expressive power. 



Our analysis centers around the expressivity of EPNN on graphs with distinct eigenvalues. This model is at least as expressive as many commonly used SGNNs \cite{zhang2024expressive}, making the analysis generalizable to these models.  Surprisingly, we find that EPNN is incomplete even on the class of graphs with distinct eigenvalues. On the positive side, it achieves completeness on this class when the eigenvectors exhibit certain sparsity patterns. Based on these theoretical insights, we propose equiEPNN, inspired by equivariant neural networks for point clouds, which attains improved expressivity on graphs with distinct eigenvalues.


Our main contributions are as summarized follows:

\begin{enumerate}
    \item We prove the incompleteness of EPNN \ref{subsec:spec-inv-1wl} on graphs with a simple spectrum.

    \item We formulate a guarantee on the completeness of EPNN on graphs with a simple spectrum based on sparsity patterns of the eigenvectors.

    \item  We introduce equiEPNN \ref{sec:spec-equi-1-wl}, a modified EPNN variant, which integrates Euclidean message passing into the feature refinement procedure.

    \item We verify that EPNN is complete on MNIST-Superpixel, equiEPNN improves its expressivity when the eigendecomposition is truncated, and equiEPNN performs perfect eigenvector canonicalization on the ZINC dataset. 
\end{enumerate}


\begin{figure}
    \centering
    \begin{tikzpicture}[
        box/.style={
            rectangle,
            draw,
            rounded corners=2pt,
            minimum width=1.8cm,
            minimum height=0.8cm,
            text width=1.8cm,
            align=center,
            inner sep=3pt,
            font=\small
        }
    ]
        
        \node[box, fill=cyan!30] (friendly) {$1$-WL};
        \node[box, fill=orange!30, right=1.5cm of friendly] (wl) {EPNN};
        \node[box, fill=violet!30, right=1.5cm of wl] (canon) {equiEPNN};

        \draw[-Latex] (friendly) -- node[above, font=\scriptsize] {$\sqsupset$} (wl);
        \draw[-Latex] (wl) -- node[above, font=\scriptsize] {$\sqsupset$} (canon);

        \coordinate (arrow_start) at ([xshift=0cm, yshift=-1.2cm]friendly.south west);
        

    \end{tikzpicture}
        \caption{Hierarchy of $1$-WL test variants. The arrows with $\sqsupset$ indicate strict inclusion relationships, meaning each variant can distinguish all graphs that the previous one can, plus additional graphs. Standard $1$-WL is the least discriminative, while equiEPNN achieves the highest discriminative power by incorporating both spectral invariant and equivariant refinement.}
    \label{fig:wl-hierarchy}
\end{figure}
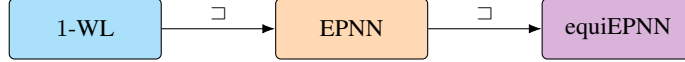
\section{Related work}
\label{sec:rel-wrk}


\subsection{Spectral invariant GNNs}
An enhancement to MPNNs and Transformer models is to incorporate spectral distances such as Random Walk, resistance, and shortest-path distances within the message passing operation \cite{li2020distance, zhang2023rethinking, mialon2021graphit, feldman2022weisfeiler}.
\citet{zhang2024expressive} have compared among spectral GNNs and the WL hierarchy by proving EPNN is strictly more powerful than $1$-WL yet strictly less powerful than $3$-WL. Despite the important result that  $3$-WL strictly bounds the expressive power of EPNN,  the large expressivity gap between $1$- and $3$-WL makes this determination difficult to conceptualize. Building on this work, \citet{gai2025homomorphism} have characterized the expressive power of EPNN via graph homomorphism counting, showing spectral invariant GNNs can homomorphism-count a class of specific tree-like graphs. Despite providing a deeper understanding of EPNN's expressive power, it remains hard to conceptualize and propose more expressive models based on it.


\subsection{Spectral canonicalization methods}

The eigenvectors of a graph are used as positional encoding to improve the expressive power of message-passing and as positional encoding for Transformer \cite{rampasek2022recipe, kreuzer2021rethinking, ma2023graph} based models. Yet, positional encoding has an inherent ambiguity problem. An eigenvector corresponding to a unique eigenvalue can be represented as itself or its negation \cite{spielman2012spectral}. Canonicalization methods \cite{canonicalization-perspective, ma2023laplacian} are used to address the ambiguity problems of eigenvectors, by choosing a unique representative for each eigenvector.

\citet{ma2023laplacian} have uncovered an inherent limitation of canonicalization methods that process each eigenspace separately, which is that they cannot canonicalize eigenvectors with nontrivial self-symmetries. These models process each eigenbasis independently to obtain an orthogonal invariant and permutation-equivariant feature, and then use these features for downstream applications. Notable examples include SignNet and BasisNet \cite{lim2023sign}, MAP \cite{ma2023laplacian} and OAP \cite{canonicalization-perspective}. \citet{canonicalization-perspective} have shown that these methods lose information when canonicalizing eigenvectors with self-symmetries, proving that the popular spectral invariant models SignNet and BasisNet are incomplete. In section \ref{subsec:exp-eigvec-canon}, we provide a canonicalization scheme that bypasses this issue, and while not provable complete on all eigenvectors, empirically it canonicalizes all eigenvectors corresponding to distinct eigenvalues, in the ZINC \cite{Irwin2012} dataset.

\subsection{Expressivity on simple spectrum graphs}
An early study on the connection between GNNs and spectral features of the underlying graph studied the expressive power of CGNs \cite{wang2022powerful}. They have proven that linear graph convolutional neural networks (GCNs) can map a graph signal to any chosen target vector, if the graph has distinct eigenvalues. Yet, this graph signal is sampled randomly and thus is not equivariant to permutations of the graph nodes, which may lead to degraded generalization, see \citet{inv-node}.

\section{Problem statement}
\label{sec:dets}
\subsection{Spectral graph decomposition}
Graphs are typically represented by a matrix $A\in \RR^{n\times n}$, where the $(i,j)$ entry of the matrix encodes the relationship between node $i$ and $j$. This matrix could be the adjacency matrix, the normalized or un-normalized graph Laplacian, or a distance or Gram Matrix where the graph nodes have some underlying geometry. 

An important principle in the design of graph neural networks is the notion of permutation invariance: since graph nodes do not come with an intrinsic order, we would like to think of a matrix $A$ and its conjugation $PAP^T $ by a permutation matrix $P\in S_n$, as being equivalent. Graph neural networks respect this invariance constraint and produce permutation invariant function $f$ satisfying $f(A)=f(PAP^T)$. One popular method to design these functions exploits the eigendecomposition of the matrix $A$. 

For the purpose of discussion, we assume that $A$ has an eigenbasis $v^{(1)},\ldots,v^{(n)} $ of vectors of norm one,  which corresponds to real eigenvalues $\lambda^{(1)},\ldots,\lambda^{(n)}$. This gives us an alternative representation of the matrix $A$ with its own interesting symmetries. Firstly, we note that
each vector $Pv^{(q)}$ will be an eigenvector of $PAP^T $ with the same eigenvalue $\lambda^{(q)}$. Secondly, if $v^{(q)}$ is an eigenvector of norm one, then so is $-v^{(q)}$. When the eigenvalues of $f$ are pairwise-distinct, then these are all the relevant ambiguities. This is called the simple spectrum case.  In the case of an eigenbasis of dimension $k$, the eigendecomposition ambiguity is defined by orthogonal transformations in $O_{k}$. In this paper, we will focus on the simple spectrum case. In this case, we define sign-invariant functions as follows

\begin{definition}[Sign Invariant functions]\label{def:invariant}
Denote 
$$\domain=\{(V,\vec{\lambda})\in \RR^{n\times K}\oplus \RR^K| \quad  \lambda_1>\lambda_2>\ldots>\lambda_K \}. $$ 
We say that $F:\domain \to \RR^m $ is \emph{sign invariant} if 
$$F(V,\vec{\lambda})=F(PVS,\vec{\lambda}), \quad \forall P\in S_n, \quad S\in \{-1,1\}^K $$
\end{definition}
We note that in this definition, $V$ represents a $n\times K$ matrix whose $K$ columns are the eigenvectors $v^{(1)},\ldots,v^{(q)} $, and the notation $S\in \{-1,1\}^K $ means that $S$ is a diagonal matrix whose diagonal is a vector in $\{-1,1\}^K $. 

The notion of sign invariant function was first introduced in \cite{lim2023sign}, and was later discussed in \cite{ma2023laplacian, canonicalization-perspective}. These papers discuss a collection  of parametric functions $\mathcal{F}=\{f_\theta(V,\vec{\lambda})| \quad \theta \in \Theta\}$, such that for all parameters $\theta$ the function $f_\theta$ is sign invariant. To understand the expressiveness of these models, we formally define the notion of completeness on simple spectrum graphs.

\begin{definition}[Sign Invariant Separation]
For $K\leq n$, let $\mathcal{F}$ denote a collection of sign invariant functions defined on $\domain$, and let $\mathcal{D}$ be a subset of  $\domain$. We say that  $\mathcal{F}$ is \textbf{complete} on  $\mathcal{D}$ if for any pair  $(V,\vec{\lambda})$ and $(U,\vec{\eta})$ in   $\mathcal{D}$, there exists a function $f \in \mathcal{F}$ such that 
    $$
 f(V,\vec{\lambda})\neq f(U,\vec{\eta}).
    $$
\end{definition}

Ideally, we would like $\mathcal{F}$ to be complete on all of the domain $\domain$. If $\mathcal{F}$ is complete, then by applying it to eigendecompostions of graphs with simple spectrum, we will obtain models which can separate all graphs with simple spectrum, up to permutation equivalence. The goal of this paper is to understand whether existing sign-invariant functions are complete.

\subsection{EPNN}\label{subsec:spec-inv-1wl}
We will focus on a large family of sign invariant functions named \emph{Eigenspace Projection GNNs (EPNN)}. This family of functions, introduced in \citet{zhang2024expressive}, was shown to generalize many spectral invariant methods such as Random Walk, resistance, and shortest-path distances \cite{li2020distance, zhang2023rethinking, mialon2021graphit, feldman2022weisfeiler}.  This method is based on a message passing like mechanism, where the spectral information is encoded by using the projection onto eigenspaces as edge features. In the simple spectrum case, this method can be formulated as follows:

For a given eigendecomposition $(V,\vec{\lambda})\in \domain $, we  we initialize a coloring for each `node' $i\in [n]$ by
\begin{equation}\label{eq:init}
h_i^{(0)} =V_i\odot V_i ,
\end{equation}
where $V_i$ is the $K$ dimensional vector $[V_i^{(1)},\ldots,V_i^{(k)}] $ obtained by sampling all eigenvectors at the $i$-th node, and $\odot$ denotes elementwise multiplication. Importantly, this initialization is sign-invariant: while the global sign of each eigenvector is ambiguous, the product of two elements of the same eigenvector is not. 

We next iteratively refine the node features via the update rule:
\begin{equation}\label{eq:wl-update-inv}
    h_i^{(t+1)} = \mathrm{UPDATE}_{(t)}\Big(h_i^{(t)},\vec{\lambda},\big\{(h_j^{(t)}, V_i \odot V_j )\mid j=1,\ldots,n\big\}\Big)
\end{equation}

Finally, we apply a global pooling operation to obtain a final permutation invariant representation
\begin{equation}\label{eq:readout}
h_{\mathrm{global}} = \mathrm{READOUT}(\{h_i^{(T)} \mid i=1,\ldots, n\})
\end{equation}
Once $\mathrm{UPDATE}_{(t)}$ and READOUT functions are determined, this procedure determines a function $f(V,\vec{\lambda})=h_{\mathrm{global}} $ which is sign-invariant as in Definition \ref{def:invariant}. The collection of all such functions obtained by all possible choices of $\mathrm{UPDATE}_{(t)}$ and READOUT functions is denoted by $\mathcal{F}_{\text{EPNN}} $.

\subsection{Equivariant EPNN}\label{sec:spec-equi-1-wl}

In \cite{gai2025homomorphism}, the authors suggest methods based on higher order WL tests to boost the expressive power of spectral message passing neural networks. The complexity of these methods is considerably higher than EPNN. In contrast, we will now suggest a method for increasing the expressive power of EPNN without significantly changing model complexity. 

Our suggestions are based on constructions from neural networks for geometric point clouds. These neural networks operate on point clouds $X \in \RR^{n\times d} $, where in many applications $d=3$ and each of the $n$ points in $\RR^3$ represents a geometric coordinate. Models for such data are required to be invariant (or equivariant) to both permutations in $S_n$ and rotations in $O(d)$. This equivariant structure is similar to, but not identical to, the situation we have for graph eigecomposition: under the simple spectrum assumption, the symmetry transformations we are interested in is a single  global permutation, and $K$ sign changes, which are rotations in $O(1)^K$. In the more general setting, we will have a single permutation and multiple rotations, whose dimension is determined by the multiplicity of each eigenvalue. 

Via this analogy, we can look at spectral models for graphs from the perspective of point cloud networks. From this perspective, EPNN resembles geometric invariant networks, such as Schnet \cite{schutt2017schnet}, which are based on simple invariant features. In contrast, \cite{joshi2023expressive} and \cite{sverdlov2024expressive} showed that, at least for point clouds, expressivity can be increased by recursively updating a rotation equivariant (in our scenario, sign equivariant) feature $v_i^{(t)} $ in parallel with the invariant feature $h_i^{(t)}$. Inspired by these observations, we suggest the following sign equivariant feature refinement procedure:

We use the same initialization $h_i^{(0)} $ as in Equation \ref{eq:init}, and we initialize the equivariant feature $v_i^{(0)} $ to $v_i^{(0)}=V_i$. We then iteratively update these two features via

\begin{align*}
    h_i^{(t+1)} &= \mathrm{UPDATE}_{(t,1)}\Big(h_i^{(t)},\vec{\lambda},\big\{(h_j^{(t)}, v_i^{(t)}\odot v_j^{(t)} )\mid j=1,\ldots,n\big\}\Big)\\
     v_i^{(t+1)} &= v_i^{(t)}+\sum_{j=1}^n  v_j^{(t)} \odot \mathrm{UPDATE}_{(t,2)}(h_i^{(t)}, h_j^{(t)},v_i^{(t)}\odot v_j^{(t)})
\end{align*}
where $\mathrm{UPDATE}_{(t,1)} $ is a multiset function, and $\mathrm{UPDATE}_{(t,2)}$ maps its input to $\RR^K $ so that the elementwise product in the equation above is well defined. 

After running this procedure for $T$ iterations, we obtain an invariant global feature $h_{\mathrm{global}} $ by aggregating the invariant node features $h_i^{(T)}$ using a READOUT function, as in \eqref{eq:readout}. This gives us a sign invariant function $f(V,\vec{\lambda})=h_{\mathrm{global}} $. We name the class of all functions obtained by running this procedure with all different choices of update and readout functions \equi. 


We note that we can obtain EPNN models by setting $\mathrm{UPDATE}_{(t,2)} $ to be the constant mapping to the zero vector. Accordingly, $\mathcal{F}_{\mathrm{equi}} $ is at least as expressive as EPNN. In Section \ref{seubsec:equi-more-express} we will show that it is strictly more expressive.

\section{On the incompleteness of spectral graph neural networks}\label{sec:on-the-incopm}

In this section, we analyze the expressive power of EPNN and equiEPNN on graphs with a simple spectrum. We first provide a counterexample to prove its incompleteness of EPNN on simple spectrum graphs.  We then show that an equiEPNN can separate the counterexample, thus proving it is strictly more expressive than EPNN. Next we provide a subset of $\domain$ on which EPNN is complete. Finally, we discuss how our results imply the incompleteness of popular spectral GNNs even in the simple spectrum case. 

\subsection{EPNN is incomplete} \label{subsec:siwl-incomp}
We first introduce a pair of non-isomorphic eigendecompositions, $(V,\vec{\lambda})  $ and $(U,\vec{\lambda})  $ in $\domain$, which EPNN cannot distinguish, that is, it assigns them the same final feature after any number of refinement steps. In this construction $n=12,K=6$, and we fix the same choice of distinct eigenvalues $\vec{\lambda}$ for both examples. To define $V,U$, we first denote the four elements of the abelian group $\{-1,1\}^2$ by
$$z_0=\begin{pmatrix}
1\\
1	
\end{pmatrix} 
, \quad
z_1=\begin{pmatrix}
	-1\\
	1	
\end{pmatrix}
,\quad
z_2=\begin{pmatrix}
	1\\
	-1	
\end{pmatrix}
,\quad
z_3=\begin{pmatrix}
	-1\\
	-1	
\end{pmatrix}
,\quad
0_2=\begin{pmatrix}
0\\
0	
\end{pmatrix}.
$$
Using these, we define $U,V $ via

\setcounter{MaxMatrixCols}{20}
$$U^T=
\begin{pmatrix}
z_0 & z_1 & z_2 & z_3 & 0_2 & 0_2 &0_2 & 0_2 &  z_0 & z_1 & z_2 &z_3\\
z_0 & z_1 & z_2 & z_3 & z_0 & z_1 &z_2 & z_3 &  0_2 & 0_2 & 0_2 & 0_2\\
0_2 & 0_2 & 0_2 & 0_2 & z_0 & z_1 &z_3 & z_2 & z_0 & z_2 & z_1 & z_3

\end{pmatrix}
$$ 

$$V^T=
\begin{pmatrix}
z_0 & z_1 & z_2 & z_3 &0_2  & 0_2 &0_2 & 0_2 &   z_0 & z_1 & z_2 &z_3\\
z_0 & z_1 & z_2 & z_3 &z_0  & z_1 &z_2 & z_3 &   0_2 & 0_2 & 0_2 & 0_2\\
0_2 & 0_2 & 0_2 & 0_2 &z_1  & z_0 &z_2 & z_3 &  z_2 & z_0 & z_3 & z_1
\end{pmatrix}
$$ 
We now show that $U,V$ are not isomorphic and cannot be separated by EPNN:

\begin{restatable}{theorem}{thmtwo}
    (Incompleteness of EPNN) 
The following statements hold:
\begin{enumerate}
    \item $U$ and $V$ are not isomorphic under the group action of $\{-1,1\}^{6}\times S_{12}$.
    \item EPNN cannot separate $U$ and $V$ after any number of iterations.
        \item $U$ and $V$ have no non-trivial automorphisms.
\end{enumerate}
Therefore, EPNN is incomplete on simple spectrum graphs. 
     
\end{restatable}





	


\begin{proof}[Proof Idea]
To show $U,V$ are not isomorphic, we note that for any pair of permutation-sign matrices taking $U$ to $V$, the first four columns of $U^T$ must be mapped the first four columns of $V^T$. The same is true for columns $5-8$ and $9-12$. Considering the first four columns, we see that any sign matrix mapping them from $U$ to $V$ will be of the form $\mathrm{diag}(z,z,z') $ for $z,z' \in \{-1,+1\}^2$. The same argument for columns $5-8$ and $9-12$ gives  sign patterns of the form  $\mathrm{diag}(z',z,z_1\cdot z) $ and  $\mathrm{diag}(z,z',z_2\cdot z) $, respectively. But there is no sign pattern satisfying these three constraints simultaneously. 

 We now explain the lack of separation of EPNN. 
 We refer to the multiset of the multiplications of a column $i$ with all the other columns, as the column $i$'s purview. In the initial step, the purview of each column in the first $4$-column block in $V^T$ and  $U^T$, is identical, as the first $4$ columns exhibit a group structure with the multiplication operation. Thus, the hidden states of the first $4$ indices of $U^T$ and $V^T$ will be identical. By similar arguments, this holds for the remaining two blocks. Thus, after a refinement step, the nodes in each block cannot distinguish among those from other blocks,  both in  $U^T$ and $V^T$. Therefore, additional refinement procedures maintain identical representations for members of each index `block' and corresponding blocks in $U^T$ and $V^T$. This implies EPNN cannot separate $U$ and $V$.
 
 A full proof of the theorem is provided in the Appendix. 
\end{proof}
We further prove that even Equivariant EPNN is incomplete over simple spectrum graphs. Or formally,

\begin{restatable}{theorem}{thmthree}
    (Incompleteness of Equivariant EPNN) 
There exist $X, Y \in \RR^{6 \times 16}$ such that the following statements hold:\begin{enumerate}
    \item $X$ and $Y$ are not isomorphic under the group action of $\{-1,1\}^{6}\times S_{16}$.
    \item Equivariant EPNN cannot separate $X$ and $Y$ after any number of iterations.
\end{enumerate}
Therefore, Equivariant EPNN is incomplete on simple spectrum graphs. 
\end{restatable}

\textbf{Remark:} In many cases we are interested in eigenvalue decompositions of symmetric matrices, in which case the columns of $V,U$ (the rows of $V^T,U^T$) should be orthonormal. While our $V,U$ do not satisfy this condition, in the Appendix we show how they can be enlarged to give an enlarged counterexample which has the same properties, and does have orthonormal columns. 

\subsection{When is EPNN complete?}

The counterexample proves that there is an inherent limit to the expressive power of contemporary spectral invariant networks. We note that in this example $U,V$ had a significant number of zero entries. We now show that when $U,V$ have less than $n$ zeros, EPNN will be complete:

\begin{theorem}[EPNN Can Distinguish Dense Graphs with Distinct Eigenvalues]\label{thm:2}

Let $\mathcal{D} \subseteq \domain$ denote the set of $(V,\vec{\lambda}) $ where $V$ has strictly less than $n$ zero-valued entries. Then EPNN is complete on $\mathcal{D}$ .
\end{theorem}
\begin{proof}
    From the pigeonhole principle, an index exists where the $i$-th row of $V$ is has no zeros. The hidden state $h_i^{(1)}$ after a signal iteration of EPNN (see \eqref{eq:wl-update-inv}) can encode the eigenvalues $\vec{\lambda} $, the squared values of each coordinate of $V_i$,  and the multiset  
    
 \begin{equation}\label{eq:achimedean}
    h_i^{(1)} = ( V_i \odot V_i ,  \big\{ V_i \odot V_j \mid j=1,\ldots,n\big\})
\end{equation}
To recover $V$ from $h_i^{(1)}$ up to symmetries, we can fix the sign ambiguity by choosing all coordinates of $V_i$ to be positive. We can then recover the remaining $V_j $ from the multiset in Equation \ref{eq:achimedean}.




\end{proof}

This uncovers the inner workings of EPNN in processing simple spectrum graphs. Essentially, each entry can be normalized to represent a group element in $O(1)$, which acts as a local frame of reference, see \cite{DymLS24} for more background, allowing us to reconstruct the eigenvectors up to sign symmetries.

\subsection{Unique node identification via EPNN}
A well-known mechanism for circumventing the limited expressive power of GNNs is by injecting unique node identifiers (IDs), which break the symmetries that hinder GNNs' separation ability \cite{loukas2020graphneuralnetworkslearn, cyclesgnns}. Popular approaches include random node initialization \cite{inv-node} and combinatorial methods \cite{dong2024rethinking}, yet they are either limited by their discontinuity or break permutation equivariance. A natural question is whether the node features from EPNN are unique after finitely many iterations? If so, we have attained node IDs that do not break equivariance and change continuously with the eigendecomposition, alleviating the deficiencies of widely-used methods. We answer this question in the affirmative, provided the eigenvectors adhere to a sparsity pattern. 

\begin{restatable}{theorem}{uni}
    (EPNN for Unique Node Identifiers)    
Let $\mathcal{D} \subseteq \domain$ denote the set of $(V,\vec{\lambda}) $ where $V$ has no automorphisms, and has at most one zero per eigenvector. Then, one iteration of EPNN with injective UPDATE and READOUT functions assigns a unique identifier to each hidden node feature.
\end{restatable}

\begin{proof}
    
    
    



    By contradiction, assume that there exist indices $i,j$ such that $h_{i}^{(1)} = h_{j}^{(1)}.$ By the definition of EPNN we have that 
    \begin{equation}\label{eq:eq-hid_states}
        \{ V_i \odot V_k \}_{k=1}^{n} =  \{ V_j \odot V_k \}_{k=1}^{n} \text{ and } V_i \odot V_i= V_j \odot V_j.
    \end{equation}
We deduce for the second equality that all coordinates $|V_i^{(q)}|=|V_j^{(q)}| $ for all $q=1,\ldots,K$. If for some $q$ we would have $V_i^{(q)}=0$ then also $V_j^{(q)}=0$, in contradiction to the assumption that both the $q$-th eigenvector has at most one zero entry. 





Now we consider two cases:
    \textit{Case 1. $\mathbf{\uq_{i}= V^{(q)}_{j}}$ for all $\mathbf{q=1,\ldots, n}.$ } In this case, swapping $i$ with $j$ will not change $V$, and thus there is a non-trivial automorphism, which is a contradiction.

    \textit{Case 2. $\exists q_2$ such that  $\mathbf{ V^{(q_2)}_{i}=  -V^{(q_2)}_{j}}.$ } We have from Equation \ref{eq:eq-hid_states}
that $$\{ (s_{i}^{(q)}\uq_{k})_{q=1}^{n} \}_{k=1}^{n} = \{ (s_{j}^{(q)}\uq_{k})_{q=1}^{n} \}_{k=1}^{n}
    $$\label{eq:eq-case-2} where $s_{i}^{(q)} \in \{\pm 1 \}$ and is defined as $\frac{V_i^{(q)}}{|V_i^{(q)}|}$ and $s_{j}^{(q)}$ is defined analogously. Note that
     $\uq_{i}, \uq_{j}\neq 0, \; \forall q=1,\ldots, n$, thus it is well-defined.

     From the set equality, it means that there exists a non-identity permutation $\sigma$ such that $s_{i}^{(q)}V_{k}^{(q)} = s_{j}^{(q)}V_{\sigma(k)}^{(q)} $ for all $k,q=1,\ldots, n.$
   Or equivalently,
    \begin{equation}\label{eq:sim-perm}
        PVS_1 =VS_2 \implies PVS_1S_2 =V
    \end{equation}

    where $S_1$ and $S_2$ are diagonal matrices with $s_i^{(q)}$ and $s_j^{(q)}$, respectively, on the diagonals. By the assumption, $S_1S_2$ is not the identity matrix and 
    thus there is a non-trivial automorphism, in contradiction to the assumption. Thus $h_i^{(1)}\neq h_j^{(1)} $ as  required.  
\end{proof}




\subsection{equiEPNN is strictly more expressive than EPNN}\label{seubsec:equi-more-express}
After establishing the incompleteness of EPNN, we show that equiEPNN  it is strictly more powerful, as it separates the pair $U$ and $V$ from Subsection \ref{subsec:siwl-incomp}, which EPNN cannot separate:

\begin{restatable}{corollary}{corStronger}
    equiEPNN (see Section \ref{sec:spec-equi-1-wl}) can separate $U$ and $V$ after $2$ iterations. Thus equiEPNN is strictly stronger than EPNN. 
\end{restatable}
\begin{proof}[Proof Idea]
We show that after a single iteration, the equivariant update step can yield new matrices $U^{(t)}, V^{(t)}, t=1$ which have no zeros. From Theorem \ref{thm:2}, we know that a single iteration of EPNN, and hence also equiEPNN, is complete for such $U^{(t)}, V^{(t)}$, and thus two iterations of equiEPNN are sufficient for separation. 
\end{proof}

\subsection{Incompleteness of spectral GNNs}\label{subsec:rl-wrld}

Theorem 1 proves that EPNN is incomplete on graphs with a simple spectrum. This spectral isomorphism test upper bounds the expressive power of many popular distance-based GNNs, which incorporate graph distances as edge features, such as Random Walk, PageRank, shortest path, or resistance distances \cite{zhang2023rethinking,li2020distance, ArnaizRodriguez2022, Velingker2023, zhang2023complete}. Therefore, an immediate corollary of Theorem 1 follows:

\begin{corollary}
Graphormer-GD \cite{zhang2023rethinking}, PRD-WL \cite{li2020distance}, DiffWire \cite{ArnaizRodriguez2022}, and Random-Walk based GNNs\cite{Velingker2023, zhang2023complete} are incomplete over graphs with a simple spectrum.
\end{corollary}

In addition to this result, in the appendix we prove that the model proposed by \citet{zhou2024stablegloballyexpressivegraph} is not universal on simple spectrum graphs, contrary to their claim.
\begin{restatable}{proposition}{propMuhan}
\label{cor:oge-incom}
    Vanilla OGE-Aug \cite{zhou2024stablegloballyexpressivegraph} is incomplete over graphs with a simple spectrum.
\end{restatable}
 




\section{Experiments}

Our goal in the experiments section is to empirically verify the theoretical results we derived. Firstly, we statistically analyze the eigenvalue multiplicity in real-world datasets, and the number of non-zero entries in the eigenvalues. Secondly, we wish to evaluate the empirical benefit of the improved expressivity of equiEPNN \ref{sec:spec-equi-1-wl} on graphs with a simple spectrum. Finally, we wish to evaluate the utility of the equivariant features derived from Spectral Equivariant $1$-WL on the task of eigenvector canonicalization \cite{canonicalization-perspective}.

\subsection{Dataset statistics}

We surveyed popular graph datasets and documented their graph spectral properties. The results are shown in Table \ref{tab:graph-stats}.  We find that 
the MNIST Superpixel \cite{super-pixel} dataset is  almost homogeneously composed of graphs with a simple spectrum, and we find that ($96.9 \%$) of the graphs in this dataset have a full row without zeros, implying that  EPNN is complete on almost all graphs. 

Other datasets, such as MUTAG, ENZYMES, PROTEINS and ZINC \cite{Irwin2012, morris2020tudataset}, contain a substantial amount of graphs with eigenvalue multiplicity $2$ and $3$. Despite this, the number of eigenspaces of dimensions $2$ and $3$ is very few per graph, averaging at around $1$ per graph. On datasets with highly symmetric graphs, such as ENZYMES and PROTEINS, the graphs do not meet the sparsity condition of Theorem \ref{thm:2}, thus EPNN will not necessarily faithfully learn the graph structure.  This exemplifies the need for more expressive models that are complete on graphs with higher maximal eigenvalue multiplicity and sparse eigenvectors.

\begin{table}[H]
\centering
\caption{Graph Statistics Analysis Across Different Datasets (Eigenvalue Tolerance: $10^{-4}$)}
\label{tab:graph-stats}
\resizebox{\textwidth}{!}{
\begin{tabular}{lccccc}
\toprule
\textbf{Dataset Name} & \textbf{MUTAG} & \textbf{ENZYMES} & \textbf{PROTEINS} & \textbf{MNIST} & \textbf{ZINC} \\
\midrule
\multicolumn{6}{l}{\textbf{Dataset Overview}} \\
Number of Graphs & 188 & 600 & 1,113 & 60,000 & 10,000 \\
\midrule
\multicolumn{6}{l}{\textbf{Eigenvalue Characteristics}} \\
Graphs with Distinct Eigenvalues & 41.5\% (78) & 34.8\% (209) & 22.1\% (246) & 99.9\% (59,950) & 40.7\% (4,072) \\
Graphs with Multiplicity 2 Eigenvalues & 58.5\% (110) & 65.2\% (391) & 77.9\% (867) & -- & 59.3\% (5,928) \\
Graphs with Multiplicity 3 Eigenvalues & 19.1\% (36) & 46.2\% (277) & 57.9\% (644) & -- & 26.2\% (2,617) \\
Avg. Number of Multiplicity 2 Eigenvalues & 0.74 & 1.01 & 1.24 & -- & -- \\
Avg. Number of Multiplicity 3 Eigenvalues & 0.26 & 0.58 & 0.71 & -- & -- \\
\midrule
\multicolumn{6}{l}{\textbf{Eigenvector Properties}} \\
Average Ratio of Zeros & 1.67 & 4.28 & 6.39 & 0.31 & 2.52 \\
Average Number of Zeros & 31.13 & 172.93 & 817.20 & 23.16 & 61.04 \\
Graphs with a Full Row & 75.0\% (141) & 35.8\% (215) & 37.1\% (413) & 96.9\% (58,077) & 64.5\% (6,447) \\
Graphs with $\leq$1 Zero per Eigenvector & 0.0\% (0) & 6.3\% (38) & 5.0\% (56) & 20.2\% (12,085) & 4.3\% (430) \\
Graphs with Total Zeros $<$ Vertices & 29.8\% (56) & 16.3\% (98) & 14.3\% (159) & 89.9\% (53,873) & 13.0\% (1,295) \\
Graphs Meeting Any Condition & 75.0\% (141) & 35.8\% (215) & 37.1\% (413) & 96.9\% (58,077) & 64.5\% (6,447) \\
\bottomrule
\end{tabular}
}
\end{table}
\subsection{MNIST Superpixel}
As a toy experiment to examine the potential benefit of using equiEPNN, 
We implemented equiEPNN via a modification of the EGNN architecture \cite{satorras2021n} and EPNN with the same architecture, but without the eigenvector update step. For precise hyperparameter configuration, see the Appendix.
\begin{wraptable}{r}{0.5\textwidth}
  \centering
  \begin{tabular}{c|cc}
       \hline
        $k$ & EPNN & EquiEPNN \\
        \hline
        3  & 48.45 $\pm$ 1.2  \% & 60.95 $\pm$ 0.9 \% \\
        8  & 85.55 $\pm$ 2.1 \% & 83.56 $\pm$ 2.5 \% \\
        16 & 90.13 $\pm$ 2.3 \% & 91.37 $\pm$ 2.2 \% \\
        \hline
  \end{tabular}
   \caption{Ablation study on MNIST Superpixel \cite{super-pixel}. Accuracy percentage comparison with deviation over $3$ trials, for different values of $K$ for EPNN and equiEPNN. }
    \label{tab:accuracy_comparison}
\end{wraptable}
 We examined the performance of both methods on the MNIST Superpixel datasets, where the task is classification of handwritten digits. We found that equiEPNN outperforms EPNN, with the same number of model parameters and hyperparameter instantiations, in the setting with few known eigenvectors. With a sufficient number of eigenvectors EPNN and equiEPNN achieve comparable results, as expected, since they are both complete on almost all graphs in MNIST Superpixel. (see k=8 and k=16 in Table \ref{tab:accuracy_comparison}.) 

Furthermore, we evaluated equiEPNN in comparison with other leading architectures, with a comparable parameter budget of $\approx 35 $K. see Table \ref{tab:zinc12k_mnist75_results}. We observe that it outperforms PPGN \citet{maron2019provably}, which has cubic complexity, and GNNML1, which also processes the eigendecomposition of the graph. ChebNet outperforms all other methods, perhaps due to its handcrafted polynomial features.  
\begin{table}[htbp]
\centering
\caption{Results on Zinc12K and MNIST-75 datasets. The values are the MSE for Zinc12K and the accuracy for MNIST-75. Edge features are not used even if they are available in the datasets. For Zinc12K, all models use node labels. For MNIST-75, the model uses superpixel intensive values and node degree as node features. Models have a budget of 30K free parameters for Zinc and 35K for MNIST.}
\label{tab:zinc12k_mnist75_results}
\begin{tabular}{l|lcc}
\hline
\textbf{Category} & \textbf{Model} & \textbf{Zinc12K} & \textbf{MNIST-75} \\
\hline
NN & MLP & $0.5869 \pm 0.025$ & $25.10\% \pm 0.12$ \\
\hline
MPNN & GCN & $0.3322 \pm 0.010$ & $52.80\% \pm 0.31$ \\
 & GAT & $0.3977 \pm 0.007$ & $82.73\% \pm 0.21$ \\
 & GIN & $0.3044 \pm 0.010$ & $75.23\% \pm 0.41$ \\
 \hline
3-WL & PPGN & $0.1589 \pm 0.007$ & $90.04\% \pm 0.54$ \\
\hline
Spectral & ChebNet & $0.3569 \pm 0.012$ & $\mathbf{92.08\% \pm 0.22}$ \\
 & GNNML1 & $0.3140 \pm 0.015$ & $84.21\% \pm 1.75$ \\
 & equiEPNN (Ours) & $\mathbf{ 0.2805 \pm 0.019}$  & 90.32 \% $\pm$ 0.7\\

\hline
\end{tabular}
\end{table}

\subsection{Realizable expressivity}

We have tested an implementation of equiEPNN on the expressivity benchmark BREC \cite{wang2023empirical}, which contains highly regular and symmetric graphs that are difficult to separate for GNNs with low expressive power. Since equiEPNN is designed to improve the expressivity of EPNN on graphs with a simple spectrum, there is no reasonable expectation for it to outperform EPNN on BREC. We tested out equiEPNN on BREC to verify our assumption, and both models indeed attain statistically equivalent performance. See the Appendix for more details.

\subsection{ZINC}
We further evaluate equiEPNN via the standard regression task on the ZINC dataset of molecular graphs (in the manuscript, we tested eigenvector canonicalization on this same dataset). ZINC (Subset) has 12000 graphs with an average of 23.16 nodes per graph. We compare ourselves to leading methods with the standard $\approx 500K$ parameter budget and find that our method attains the best results:

Graph regression results on ZINC with a parameter budget of ~500K

\begin{table}[h]
\centering
\caption{Results on ZINC.}
\label{tab:zinc_results}
\begin{tabular}{l|c}
\hline
method & test MAE \\
\hline
GIN  & 0.526±0.051 \\
GraphSage  & 0.398±0.002 \\
GCN   & 0.384±0.007 \\
GCN   & 0.367±0.011 \\
GatedGCN-PE  & 0.214±0.006 \\
MPNN (sum)  & 0.145±0.007 \\
PNA   & 0.142±0.010 \\
GT  & 0.226±0.014 \\
SAN  & 0.139±0.006 \\
Graphormer$_{\text{SLIM}}$  & 0.122±0.006 \\
\hline
MPNN & .138 ± .006  \\
EPNN & .103 ± .006 \\
equiEPNN (Ours)  & \textbf{0.99 $\pm$ 0.001}\\
\hline
Subgraph GNN & .110 ± .007  \\
Local 2-GNN & .069 ± .001  \\
\end{tabular}
\end{table}




\subsection{Eigenvector canonicalization}\label{subsec:exp-eigvec-canon}
Positional encoding is a cornerstone of graph learning using Transformer architectures, yet they suffer from the sign ambiguity problem \cite{dwivedi2023benchmarking}. It can be resolved by eigenvector canonicalization, which involves choosing a unique representation of each eigenvector. Yet, an inherent limitation of current canonicalization methods is that they are unable to canonicalize eigenvectors with nontrivial self-symmetries, often called uncanonicalizable eigenvectors \cite{ma2023laplacian, canonicalization-perspective}.

\begin{wraptable}{r}{0.6\textwidth}
  \centering

  \caption{Uncanonicalizable Graph Eigenvectors in ZINC (Subset) \cite{Irwin2012} as percentage of total eigenvectors of eigen-space dimension 1.}
  \begin{tabular}{lc}
    \hline
    \textbf{Property} & \textbf{Percentage (\%)} \\
    \hline
    
    Sum to $0$ & 11.15 \% \\
    Uncanonicalizable &  10.93 \%\\
    equiEPNN output sum to $0$  & 0.0 \% \\
    Uncanonicalizable after equiEPNN & 0.0 \% \\
    
    \hline
  \end{tabular}
  \label{tab:graph_canonicalization}
\end{wraptable}

To overcome this limitation, we devise a method to choose a canonical representation of the original eigenvectors via the equivariant output of equiEPNN. The only requirement is that each vector in the equivariant output does not sum to $0$, which does not occur in ZINC. For more details, see the Appendix. 


We test our hypothesis on a popular benchmark ZINC \cite{Irwin2012}, and find that essentially all the vectors in the equivariant output are canonicalizable, in contrast to the vectors from the eigendecomposition, where $10 \%$ of them are uncanonicalizable. Furthermore, we devise a way to choose a canonical representation of the original eigenvectors via the equivariant output and describe this in the Appendix. The results are shown in Table \ref{tab:graph_canonicalization}.




\section{Future Work}\label{sec:conc}



A key future goal is to achieve completeness on graphs with simple spectra and higher eigenvalue multiplicities. One interesting direction is to use higher-order point cloud networks to process the eigenvectors \cite{zhou2024stablegloballyexpressivegraph}. We have shown that treating each eigenspace as a separate entity does not lead to universality (Subsection \ref{subsec:rl-wrld}). Thus, these high-order networks should process the eigenvectors as a single entity, but remain invariant only to the sign and basis symmetries.

\bibliographystyle{plainnat}
\bibliography{bibfile}

\appendix
\clearpage
\section*{Appendix}
\startcontents[appendix]
\printcontents[appendix]{}{1}{\setcounter{tocdepth}{2}}
\section{Proofs}

\subsection{Proof of Incompleteness of EPNN}
\thmtwo*
\proof
For convenience, we recall the definitions of the point clouds $U$ and $V$:

 We denoted the four elements of the abelian group $\{-1,1\}^2$, and the zero vector in $\RR^2$, by
$$z_0=\begin{pmatrix}
1\\
1	
\end{pmatrix} 
, \quad
z_1=\begin{pmatrix}
	-1\\
	1	
\end{pmatrix}
,\quad
z_2=\begin{pmatrix}
	1\\
	-1	
\end{pmatrix}
,\quad
z_3=\begin{pmatrix}
	-1\\
	-1	
\end{pmatrix}
,\quad
0_2=\begin{pmatrix}
0\\
0	
\end{pmatrix}.
$$
Using these, we define $U,V $ via

\setcounter{MaxMatrixCols}{20}
$$U^T=
\begin{pmatrix}
z_0 & z_1 & z_2 & z_3 & 0_2 & 0_2 &0_2 & 0_2 &  z_0 & z_1 & z_2 &z_3\\
z_0 & z_1 & z_2 & z_3 & z_0 & z_1 &z_2 & z_3 &  0_2 & 0_2 & 0_2 & 0_2\\
0_2 & 0_2 & 0_2 & 0_2 & z_0 & z_1 &z_3 & z_2 & z_0 & z_2 & z_1 & z_3

\end{pmatrix}
$$ 

$$V^T=
\begin{pmatrix}
z_0 & z_1 & z_2 & z_3 &0_2  & 0_2 &0_2 & 0_2 &   z_0 & z_1 & z_2 &z_3\\
z_0 & z_1 & z_2 & z_3 &z_0  & z_1 &z_2 & z_3 &   0_2 & 0_2 & 0_2 & 0_2\\
0_2 & 0_2 & 0_2 & 0_2 &z_1  & z_0 &z_2 & z_3 &  z_2 & z_0 & z_3 & z_1
\end{pmatrix}
$$

\newpage
\section*{NeurIPS Paper Checklist}

\begin{enumerate}

\item {\bf Claims}
    \item[] Question: Do the main claims made in the abstract and introduction accurately reflect the paper's contributions and scope?
    \item[] Answer: \answerYes{}
    \item[] Justification: Our claims are properly stated in the abstract within their scope. We diligently wrote the assumptions. The experiments are aligned with the claims.
    \item[] Guidelines:
    \begin{itemize}
        \item The answer NA means that the abstract and introduction do not include the claims made in the paper.
        \item The abstract and/or introduction should clearly state the claims made, including the contributions made in the paper and important assumptions and limitations. A No or NA answer to this question will not be perceived well by the reviewers. 
        \item The claims made should match theoretical and experimental results, and reflect how much the results can be expected to generalize to other settings. 
        \item It is fine to include aspirational goals as motivation as long as it is clear that these goals are not attained by the paper. 
    \end{itemize}

\item {\bf Limitations}
    \item[] Question: Does the paper discuss the limitations of the work performed by the authors?
    \item[] Answer: \answerYes{} 
    \item[] Justification: We claim that we have not fully determined when completeness on simple spectrum graphs is achieved.
    \item[] Guidelines:
    \begin{itemize}
        \item The answer NA means that the paper has no limitation while the answer No means that the paper has limitations, but those are not discussed in the paper. 
        \item The authors are encouraged to create a separate "Limitations" section in their paper.
        \item The paper should point out any strong assumptions and how robust the results are to violations of these assumptions (e.g., independence assumptions, noiseless settings, model well-specification, asymptotic approximations only holding locally). The authors should reflect on how these assumptions might be violated in practice and what the implications would be.
        \item The authors should reflect on the scope of the claims made, e.g., if the approach was only tested on a few datasets or with a few runs. In general, empirical results often depend on implicit assumptions, which should be articulated.
        \item The authors should reflect on the factors that influence the performance of the approach. For example, a facial recognition algorithm may perform poorly when image resolution is low or images are taken in low lighting. Or a speech-to-text system might not be used reliably to provide closed captions for online lectures because it fails to handle technical jargon.
        \item The authors should discuss the computational efficiency of the proposed algorithms and how they scale with dataset size.
        \item If applicable, the authors should discuss possible limitations of their approach to address problems of privacy and fairness.
        \item While the authors might fear that complete honesty about limitations might be used by reviewers as grounds for rejection, a worse outcome might be that reviewers discover limitations that aren't acknowledged in the paper. The authors should use their best judgment and recognize that individual actions in favor of transparency play an important role in developing norms that preserve the integrity of the community. Reviewers will be specifically instructed to not penalize honesty concerning limitations.
    \end{itemize}

\item {\bf Theory assumptions and proofs}
    \item[] Question: For each theoretical result, does the paper provide the full set of assumptions and a complete (and correct) proof?
    \item[] Answer: \answerYes{} 
    \item[] Justification: All assumptions are stated, we provide proof ideas and full proofs are provided in the appendix.
    \item[] Guidelines:
    \begin{itemize}
        \item The answer NA means that the paper does not include theoretical results. 
        \item All the theorems, formulas, and proofs in the paper should be numbered and cross-referenced.
        \item All assumptions should be clearly stated or referenced in the statement of any theorems.
        \item The proofs can either appear in the main paper or the supplemental material, but if they appear in the supplemental material, the authors are encouraged to provide a short proof sketch to provide intuition. 
        \item Inversely, any informal proof provided in the core of the paper should be complemented by formal proofs provided in appendix or supplemental material.
        \item Theorems and Lemmas that the proof relies upon should be properly referenced. 
    \end{itemize}

    \item {\bf Experimental result reproducibility}
    \item[] Question: Does the paper fully disclose all the information needed to reproduce the main experimental results of the paper to the extent that it affects the main claims and/or conclusions of the paper (regardless of whether the code and data are provided or not)?
    \item[] Answer: \answerYes{} 
    \item[] Justification: All code is provided in the supplementary mateial and configurations and instructions as well. 
    \item[] Guidelines:
    \begin{itemize}
        \item The answer NA means that the paper does not include experiments.
        \item If the paper includes experiments, a No answer to this question will not be perceived well by the reviewers: Making the paper reproducible is important, regardless of whether the code and data are provided or not.
        \item If the contribution is a dataset and/or model, the authors should describe the steps taken to make their results reproducible or verifiable. 
        \item Depending on the contribution, reproducibility can be accomplished in various ways. For example, if the contribution is a novel architecture, describing the architecture fully might suffice, or if the contribution is a specific model and empirical evaluation, it may be necessary to either make it possible for others to replicate the model with the same dataset, or provide access to the model. In general. releasing code and data is often one good way to accomplish this, but reproducibility can also be provided via detailed instructions for how to replicate the results, access to a hosted model (e.g., in the case of a large language model), releasing of a model checkpoint, or other means that are appropriate to the research performed.
        \item While NeurIPS does not require releasing code, the conference does require all submissions to provide some reasonable avenue for reproducibility, which may depend on the nature of the contribution. For example
        \begin{enumerate}
            \item If the contribution is primarily a new algorithm, the paper should make it clear how to reproduce that algorithm.
            \item If the contribution is primarily a new model architecture, the paper should describe the architecture clearly and fully.
            \item If the contribution is a new model (e.g., a large language model), then there should either be a way to access this model for reproducing the results or a way to reproduce the model (e.g., with an open-source dataset or instructions for how to construct the dataset).
            \item We recognize that reproducibility may be tricky in some cases, in which case authors are welcome to describe the particular way they provide for reproducibility. In the case of closed-source models, it may be that access to the model is limited in some way (e.g., to registered users), but it should be possible for other researchers to have some path to reproducing or verifying the results.
        \end{enumerate}
    \end{itemize}

\item {\bf Open access to data and code}
    \item[] Question: Does the paper provide open access to the data and code, with sufficient instructions to faithfully reproduce the main experimental results, as described in supplemental material?
    \item[] Answer: \answerYes{} 
    \item[] Justification: All code is reproducible and provided openly with instructions.
    \item[] Guidelines:
    \begin{itemize}
        \item The answer NA means that paper does not include experiments requiring code.
        \item Please see the NeurIPS code and data submission guidelines (\url{https://nips.cc/public/guides/CodeSubmissionPolicy}) for more details.
        \item While we encourage the release of code and data, we understand that this might not be possible, so “No” is an acceptable answer. Papers cannot be rejected simply for not including code, unless this is central to the contribution (e.g., for a new open-source benchmark).
        \item The instructions should contain the exact command and environment needed to run to reproduce the results. See the NeurIPS code and data submission guidelines (\url{https://nips.cc/public/guides/CodeSubmissionPolicy}) for more details.
        \item The authors should provide instructions on data access and preparation, including how to access the raw data, preprocessed data, intermediate data, and generated data, etc.
        \item The authors should provide scripts to reproduce all experimental results for the new proposed method and baselines. If only a subset of experiments are reproducible, they should state which ones are omitted from the script and why.
        \item At submission time, to preserve anonymity, the authors should release anonymized versions (if applicable).
        \item Providing as much information as possible in supplemental material (appended to the paper) is recommended, but including URLs to data and code is permitted.
    \end{itemize}

\item {\bf Experimental setting/details}
    \item[] Question: Does the paper specify all the training and test details (e.g., data splits, hyperparameters, how they were chosen, type of optimizer, etc.) necessary to understand the results?
    \item[] Answer: \answerYes{} 
    \item[] Justification: All configurations are described in supplementary material.
    \item[] Guidelines:
    \begin{itemize}
        \item The answer NA means that the paper does not include experiments.
        \item The experimental setting should be presented in the core of the paper to a level of detail that is necessary to appreciate the results and make sense of them.
        \item The full details can be provided either with the code, in appendix, or as supplemental material.
    \end{itemize}

\item {\bf Experiment statistical significance}
    \item[] Question: Does the paper report error bars suitably and correctly defined or other appropriate information about the statistical significance of the experiments?
    \item[] Answer: \answerYes{} 
    \item[] Justification: When statistical significance is clear, we mention; otherwise we claim it is only comparable.
    \item[] Guidelines:
    \begin{itemize}
        \item The answer NA means that the paper does not include experiments.
        \item The authors should answer "Yes" if the results are accompanied by error bars, confidence intervals, or statistical significance tests, at least for the experiments that support the main claims of the paper.
        \item The factors of variability that the error bars are capturing should be clearly stated (for example, train/test split, initialization, random drawing of some parameter, or overall run with given experimental conditions).
        \item The method for calculating the error bars should be explained (closed form formula, call to a library function, bootstrap, etc.)
        \item The assumptions made should be given (e.g., Normally distributed errors).
        \item It should be clear whether the error bar is the standard deviation or the standard error of the mean.
        \item It is OK to report 1-sigma error bars, but one should state it. The authors should preferably report a 2-sigma error bar than state that they have a 96\% CI, if the hypothesis of Normality of errors is not verified.
        \item For asymmetric distributions, the authors should be careful not to show in tables or figures symmetric error bars that would yield results that are out of range (e.g. negative error rates).
        \item If error bars are reported in tables or plots, The authors should explain in the text how they were calculated and reference the corresponding figures or tables in the text.
    \end{itemize}

\item {\bf Experiments compute resources}
    \item[] Question: For each experiment, does the paper provide sufficient information on the computer resources (type of compute workers, memory, time of execution) needed to reproduce the experiments?
    \item[] Answer: \answerYes{} 
    \item[] Justification: All computer resources needed are mentioned in the supplementary material.
    \item[] Guidelines:
    \begin{itemize}
        \item The answer NA means that the paper does not include experiments.
        \item The paper should indicate the type of compute workers CPU or GPU, internal cluster, or cloud provider, including relevant memory and storage.
        \item The paper should provide the amount of compute required for each of the individual experimental runs as well as estimate the total compute. 
        \item The paper should disclose whether the full research project required more compute than the experiments reported in the paper (e.g., preliminary or failed experiments that didn't make it into the paper). 
    \end{itemize}
    
\item {\bf Code of ethics}
    \item[] Question: Does the research conducted in the paper conform, in every respect, with the NeurIPS Code of Ethics \url{https://neurips.cc/public/EthicsGuidelines}?
    \item[] Answer: \answerYes{} 
    \item[] Justification: We fully abide by the NeurIPS Code of Ethics and preserve anonymity.
    \item[] Guidelines:
    \begin{itemize}
        \item The answer NA means that the authors have not reviewed the NeurIPS Code of Ethics.
        \item If the authors answer No, they should explain the special circumstances that require a deviation from the Code of Ethics.
        \item The authors should make sure to preserve anonymity (e.g., if there is a special consideration due to laws or regulations in their jurisdiction).
    \end{itemize}

\item {\bf Broader impacts}
    \item[] Question: Does the paper discuss both potential positive societal impacts and negative societal impacts of the work performed?
    \item[] Answer: \answerNA{} 
    \item[] Justification: This is theoretical research that does not affect society. 
    \item[] Guidelines:
    \begin{itemize}
        \item The answer NA means that there is no societal impact of the work performed.
        \item If the authors answer NA or No, they should explain why their work has no societal impact or why the paper does not address societal impact.
        \item Examples of negative societal impacts include potential malicious or unintended uses (e.g., disinformation, generating fake profiles, surveillance), fairness considerations (e.g., deployment of technologies that could make decisions that unfairly impact specific groups), privacy considerations, and security considerations.
        \item The conference expects that many papers will be foundational research and not tied to particular applications, let alone deployments. However, if there is a direct path to any negative applications, the authors should point it out. For example, it is legitimate to point out that an improvement in the quality of generative models could be used to generate deepfakes for disinformation. On the other hand, it is not needed to point out that a generic algorithm for optimizing neural networks could enable people to train models that generate Deepfakes faster.
        \item The authors should consider possible harms that could arise when the technology is being used as intended and functioning correctly, harms that could arise when the technology is being used as intended but gives incorrect results, and harms following from (intentional or unintentional) misuse of the technology.
        \item If there are negative societal impacts, the authors could also discuss possible mitigation strategies (e.g., gated release of models, providing defenses in addition to attacks, mechanisms for monitoring misuse, mechanisms to monitor how a system learns from feedback over time, improving the efficiency and accessibility of ML).
    \end{itemize}
    
\item {\bf Safeguards}
    \item[] Question: Does the paper describe safeguards that have been put in place for responsible release of data or models that have a high risk for misuse (e.g., pretrained language models, image generators, or scraped datasets)?
    \item[] Answer: \answerNA{} 
    \item[] Justification: There is no risk, all datasets have no risk and are widely used.
    \item[] Guidelines:
    \begin{itemize}
        \item The answer NA means that the paper poses no such risks.
        \item Released models that have a high risk for misuse or dual-use should be released with necessary safeguards to allow for controlled use of the model, for example by requiring that users adhere to usage guidelines or restrictions to access the model or implementing safety filters. 
        \item Datasets that have been scraped from the Internet could pose safety risks. The authors should describe how they avoided releasing unsafe images.
        \item We recognize that providing effective safeguards is challenging, and many papers do not require this, but we encourage authors to take this into account and make a best faith effort.
    \end{itemize}

\item {\bf Licenses for existing assets}
    \item[] Question: Are the creators or original owners of assets (e.g., code, data, models), used in the paper, properly credited and are the license and terms of use explicitly mentioned and properly respected?
    \item[] Answer: \answerYes{} 
    \item[] Justification: All the results by other researchers are cited, stated and given credit fully.
    \item[] Guidelines:
    \begin{itemize}
        \item The answer NA means that the paper does not use existing assets.
        \item The authors should cite the original paper that produced the code package or dataset.
        \item The authors should state which version of the asset is used and, if possible, include a URL.
        \item The name of the license (e.g., CC-BY 4.0) should be included for each asset.
        \item For scraped data from a particular source (e.g., website), the copyright and terms of service of that source should be provided.
        \item If assets are released, the license, copyright information, and terms of use in the package should be provided. For popular datasets, \url{paperswithcode.com/datasets} has curated licenses for some datasets. Their licensing guide can help determine the license of a dataset.
        \item For existing datasets that are re-packaged, both the original license and the license of the derived asset (if it has changed) should be provided.
        \item If this information is not available online, the authors are encouraged to reach out to the asset's creators.
    \end{itemize}

\item {\bf New assets}
    \item[] Question: Are new assets introduced in the paper well documented and is the documentation provided alongside the assets?
    \item[] Answer: \answerYes{} 
    \item[] Justification: All code is reproducible with instructions.
    \item[] Guidelines:
    \begin{itemize}
        \item The answer NA means that the paper does not release new assets.
        \item Researchers should communicate the details of the dataset/code/model as part of their submissions via structured templates. This includes details about training, license, limitations, etc. 
        \item The paper should discuss whether and how consent was obtained from people whose asset is used.
        \item At submission time, remember to anonymize your assets (if applicable). You can either create an anonymized URL or include an anonymized zip file.
    \end{itemize}

\item {\bf Crowdsourcing and research with human subjects}
    \item[] Question: For crowdsourcing experiments and research with human subjects, does the paper include the full text of instructions given to participants and screenshots, if applicable, as well as details about compensation (if any)? 
    \item[] Answer: \answerNA{} 
    \item[] Justification: No human subjects are involved in this research.
    \item[] Guidelines:
    \begin{itemize}
        \item The answer NA means that the paper does not involve crowdsourcing nor research with human subjects.
        \item Including this information in the supplemental material is fine, but if the main contribution of the paper involves human subjects, then as much detail as possible should be included in the main paper. 
        \item According to the NeurIPS Code of Ethics, workers involved in data collection, curation, or other labor should be paid at least the minimum wage in the country of the data collector. 
    \end{itemize}

\item {\bf Institutional review board (IRB) approvals or equivalent for research with human subjects}
    \item[] Question: Does the paper describe potential risks incurred by study participants, whether such risks were disclosed to the subjects, and whether Institutional Review Board (IRB) approvals (or an equivalent approval/review based on the requirements of your country or institution) were obtained?
    \item[] Answer: \answerNA{} 
    \item[] Justification: No research with human subjects.
    \item[] Guidelines:
    \begin{itemize}
        \item The answer NA means that the paper does not involve crowdsourcing nor research with human subjects.
        \item Depending on the country in which research is conducted, IRB approval (or equivalent) may be required for any human subjects research. If you obtained IRB approval, you should clearly state this in the paper. 
        \item We recognize that the procedures for this may vary significantly between institutions and locations, and we expect authors to adhere to the NeurIPS Code of Ethics and the guidelines for their institution. 
        \item For initial submissions, do not include any information that would break anonymity (if applicable), such as the institution conducting the review.
    \end{itemize}

\item {\bf Declaration of LLM usage}
    \item[] Question: Does the paper describe the usage of LLMs if it is an important, original, or non-standard component of the core methods in this research? Note that if the LLM is used only for writing, editing, or formatting purposes and does not impact the core methodology, scientific rigorousness, or originality of the research, declaration is not required.
    \item[] Answer: \answerNA{} 
    \item[] Justification: No usage of LLMs in core method.
    \item[] Guidelines:
    \begin{itemize}
        \item The answer NA means that the core method development in this research does not involve LLMs as any important, original, or non-standard components.
        \item Please refer to our LLM policy (\url{https://neurips.cc/Conferences/2025/LLM}) for what should or should not be described.
    \end{itemize}

\end{enumerate}





We first prove:
$\mathbf{2.}$  the in-separation of $U$ and $V$ by EPNN. 

Observe the purview of node $i$ of $U$ after the first refinement step of EPNN:

\begin{equation}\label{eq:pur-one-iter}
    h_i^{(1)}(U) = ( U_i \odot U_i, \{ U_i \odot U_j \mid  j\in [10] \})    
\end{equation}

We will show that point clouds can be paritioned into `blocks' such that each point in the block obtains the same hidden state. This block structure is recognized by viewing each point as group element and each block as a multiplicative group. We will then show that this multiplicative group structure allows us to prove the inseparationof EPNN.

Concretely, our proof proceeds as follows:

\begin{enumerate}
    \item The column entries of $U$ and $U$ can be partitioned into $3$ blocks : $B_1 \triangleq \{1,2,3,4\}, B_2 \triangleq \{5,6,7,8\}$, and $B_3 \triangleq \{ 9, 10, 11, 12\}$, such $h^{(1)}_i(U)$ and $h^{(1)}_j(U)$ are identical for every $i,j \in B_k$, $k=1,2,3.$

    \item It holds that $h^{(1)}_i(U) =  h^{(1)}_i(V)$ for every $i=1,2,\ldots, 12.$
    \item For any $t \in \mathbb{N}$,  $h^{(t)}_i(U) =  h^{(t)}_i(V)$ for every $i=1,2,\ldots, 12.$
\end{enumerate}

1. We first focus on $B_1$ and then extend the argument to $B_2$ and $B_3$.

Since the  elements $ U_j$ for $j=1,2,3,4$ admit a multiplicative group structure, then for every $i=1,2,3,4$, the respective entries $U_i \odot U_j$ for $j\in [4]$ are identical (closure of groups.) 

For $j = 5,6,7,8$ and $i=1,2,3,4$, the entries of the products $V_i \odot V_j$, are are zeros in two row entries and the non-zero entries in the remaining row, each element of the group $\mathbb{Z}_2^{2} \cong  \{ z_0, z_1,z_2, z_3 \}$ appears exactly once in the non-zero entries of the products, as it holds that $z_i\mathbb{Z}_2^{2} = \mathbb{Z}_2^{2}.$

Analogously, we can extend this argument to $j = 9,10,11,12$ and $i=1,2,3,4.$

This means that $h^{(1)}_i(U)$ and $h^{(1)}_j(U)$ are identical for every $i,j \in B_1$.

Since, by definition of $U$ and symmetry, each four-index quadruple $B_1 \triangleq 1,2,3,4$, $B_2 \triangleq 5,6,7,8$, and $B_3 \triangleq 9,10,11,12$ is a multiplicative group, the analysis for the hidden states of the indices in $B_1$ holds for $B_2$ and $B_3$. This concludes item $1.$

$2.$ Up to now, we proved for indices $i,j \in B_k$ for $k=1,2,3$, it holds that $h_{i}^{(1)}=h_{j}^{(1)}$. It remains to be proven that these hidden states are equivalent in both point clouds to conclude step $2.$

Since the point cloud $V^T$ is derived from $U^T$ by multiplying the columns in $B_2$ by $\mathrm{diag}(z_0,z_0,z_1)$ and the columns in $B_3$ by $\mathrm{diag}(z_0,z_0,z_2)$,  the purview (see Equation \ref{eq:pur-one-iter}) of each index is identical in both point clouds, since $z_2\cdot z_2 = z_1\cdot z_1 =z_0$ which is the identity element, thus by definition of EPNN, this modification that maps $U^T$ to $V^T$ doesn't affect the pairwise multiplications in Equation \ref{eq:pur-one-iter}.

$3.$ To prove this step, we only need to show that the hidden states remain identical within each block, since the fact that they are identical across the point clouds stems from the same justification of step $2.$

In the second update step, the arguments of Step $1$ remain identical. Still, now we have updated hidden node information, but the hidden node information is identical across nodes belonging to the same block. Therefore, the only information this refinement yields is the categorization of nodes into blocks. Yet this information is already known in the initialized hidden states,$\{ h_i^{(0)} \; i=1, \ldots, n \}$, since the zero entries of multiplication $h_i^{(0)} = V_i \odot V_i$ determine the block that $i$ belongs to. Therefore, the hidden states don't supply the network with any supplementary information other than the initialization $h_{i}^{(0)}=V_i \odot V_i$. Thus, after a second refinement step, the hidden states remain identical within each block, as they have after the first refinement step. Moreover, the corresponding hidden states of the two point clouds also remain equivalent due to the arguments in Step $2$, which remain analogous, as the hidden states after a refinement only assign each node its respective block membership, which is exactly the information given in the first update step. This argument can then be applied recursively to any number of update steps.

In conclusion, we have shown that for any $t\in \mathbb{N}$, the hidden states of both point clouds are identical (in corresponding indices), therefore after a permutation invariant readout, we obtain the same output.

We now prove 
$\mathbf{3.}$
\textit{$U$ and $V$ have no nontrivial automorphisms.} 
To show $U,V$ are not isomorphic, we note that for any pair of permutation-sign matrices taking $U$ to $V$, the first four columns of $U^T$ must be mapped the first four columns of $V^T$. The same is true for columns $5-8$ and $9-12$. Considering the first four columns, we see that any sign matrix mapping them from $U$ to $V$ will be of the form $\mathrm{diag}(z,z,z') $ for $z,z' \in \{-1,+1\}^2$. The same argument for columns $5-8$ and $9-12$ gives  sign patterns of the form  $\mathrm{diag}(z',z,z_1\cdot z) $ and  $\mathrm{diag}(z,z',z_2\cdot z) $, respectively. But there is no sign pattern satisfying these three constraints simultaneously.

 The automorphism group of this extended eigendecomposition is contained within that of $U$ and $U$, respectively, and thus is also only the trivial group.

The proof of 
$\mathbf{1.}$ which states that $U$ and $V$ are not isomorphic, 
is analogous to the proof of $\mathbf{3}$, and yields that the only sign pattern taking each point cloud to itself is $(z_0,z_0,z_0)$, which implies each point cloud has only a trivial automorphism..

\subsection{Extension to orthonormal counterexamples}
The rows of the above point clouds $U,V$ are not orthonormal. Thus they are not eigenvectors of an eigendecomposition of a symmetric matrix. We fix this misalignment via the following `orthogonalization' matrices:

Taking $\tilde{U} $ to be a concantenation of the previous $U$ and $\hat U$ defined by
$$  \hat{U}^T=
\begin{pmatrix}
2z_0 & 2z_1 & 2z_2 & 2z_3 & 0_2 & 0_2 &0_2 & 0_2 &  2z_0 & 2z_1 & 2z_2 &2z_3\\
-\frac{z_0}{2} & -\frac{z_1}{2} & -\frac{z_2}{2} & -\frac{z_3}{2} & 2z_0 & 2z_1 &2z_2 & 2z_3 &  0_2 & 0_2 & 0_2 & 0_2\\
0_2 & 0_2 & 0_2 & 0_2 & -\frac{z_0}{2} & -\frac{z_1}{2} &-\frac{z_3}{2} & -\frac{z_2}{2} & -\frac{z_0}{2} & -\frac{z_2}{2} & -\frac{z_1}{2} & -\frac{z_3}{2}
\end{pmatrix}
$$ 

Then take 
$\tilde{V} $ to be a concantenation of the previous $V$ and $\hat V$ defined by
$$\hat{V}^T=
\begin{pmatrix}
2z_0 & 2z_1 & 2z_2 & 2z_3 & 0_2 & 0_2 &0_2 & 0_2 &  2z_0 & 2z_1 & 2z_2 &2z_3\\
-\frac{z_0}{2} & -\frac{z_1}{2} & -\frac{z_2}{2} & -\frac{z_3}{2} & 2z_0 & 2z_1 &2z_2 & 2z_3 &  0_2 & 0_2 & 0_2 & 0_2\\
0_2 & 0_2 & 0_2 & 0_2 & -\frac{z_1}{2} & -\frac{z_0}{2} &-\frac{z_2}{2} & -\frac{z_3}{2} & -\frac{z_2}{2} & -\frac{z_0}{2} & -\frac{z_3}{2} & -\frac{z_1}{2}
\end{pmatrix}
$$ 
The columns of $
\tilde U$ and $\tilde V$ are now orthogonal, and they can be made to have unit norm by normalizing each column. As these extensions exhibit the same symmetries of $U$ and $V$, respectively, analogous arguments to the proof of inseparation of $U$ and $V$ by EPNN (Theorem \ref{thm:2}) will apply to this new pair $\tilde{U},\tilde{V}$. Therefore, EPNN cannot distinguish $
\tilde U$ and $\tilde V$.

\subsection{Proofs for implications for real-world GNNs}

\propMuhan*

\begin{proof}
    The method proposed by \citet{zhou2024stablegloballyexpressivegraph} consists of a permutation equivariant and orthogonal invariant function. We will show that a counterexample by \cite{canonicalization-perspective} also applies to this network.

    Vanilla PGE-Aug relies on a permutation equivariant and orthogonal invariant set encoding to process each eigenspace separate. Lets revisit their separately definitions and theorem:

\begin{definition}[O(p)-invariant universal representation \cite{zhou2024stablegloballyexpressivegraph}]\label{def:wrong}
Let $f : \bigcup_{n=0}^{\infty} \mathbb{R}^{n \times p} \to \bigcup_{n=0}^{\infty} \mathbb{R}^n$. Given an input $V \in \mathbb{R}^{n \times p}$, $f$ outputs a vector $f(V) \in \mathbb{R}^n$. The function $f$ is said to be an $O(p)$-invariant universal representation if given $V, V' \in \mathbb{R}^{n \times p}$ and $P \in S_n$, the following two conditions are equivalent: 
\begin{enumerate}
    \item[(i)] $f(V) = Pf(V')$; 
    \item[(ii)] $\exists Q \in O(p)$, such that $V = PV'Q$.
\end{enumerate}
\end{definition}

\begin{definition}[Universal set representation \cite{zhou2024stablegloballyexpressivegraph}]
Let $X$ be a non-empty set. A function $f : 2^X \to \mathbb{R}$ is said to be a universal set representation if $\forall X_1, X_2 \in 2^X$, $f(X_1) = f(X_2)$ if and only if the two sets $X_1$ and $X_2$ are equal.
\end{definition}

\textbf{Proposition 3.5} (\citet{zhou2024stablegloballyexpressivegraph})\label{prop:universal-gnn}
For each $p = 1, 2, \ldots$, let $f_p$ be an $O(p)$-invariant universal representation function. Further let $g : 2^{\mathbb{R}^3} \to \mathbb{R}$ be a universal set representation. Then the following function
\begin{equation}
r(G, X_G) = \text{GNN}\left(A_G, \text{concat}\left[X_G, g\left(\{\text{concat}[\mu_j \mathbf{1}_n, \lambda_j \mathbf{1}_n, f_{\mu_j}(V_j)]\}_{j=1}^K\right)\right]\right)
\end{equation}
is a universal representation. Here $n = |V(G)|$, $((\lambda_1, \mu_1), \ldots, (\lambda_K, \mu_K))$ is the spectrum of $G$, and $V_j \in \mathbb{R}^{n \times \mu_j}$ are the $\mu_j$ mutually orthogonal normalized eigenvectors of $L_G$ corresponding to $\lambda_j$. We denote $\mathbf{1}_n$ an all-1 vector of shape $n \times 1$. GNN is a maximally expressive MPNN.

Then \citet{zhou2024stablegloballyexpressivegraph} propose the following graph neural network:

\textbf{Definition 3.6} (Vanilla OGE-Aug). Let $f_p$ be an $O(p)$-invariant universal representation, for each $p = 1, 2, \ldots$, and $g : 2^{\mathbb{R}^3} \to \mathbb{R}$ be a universal set representation. Define $Z : \mathcal{G} \to \bigcup_{n=1}^\infty \mathbb{R}^n$ as
\begin{equation}
Z(G) = g\left(\left\{\text{concat}\left[\mu_j \mathbf{1}_{|V(G)|}, \lambda_j \mathbf{1}_{|V(G)|}, f_{\mu_j}(V_j)\right]\right\}_{j=1}^K\right), \tag{5}
\end{equation}
in which the notations follow Proposition 3.5. For $G \in \mathcal{G}$, $Z(G)$ is called a \textbf{vanilla orthogonal group equivariant augmentation}, or \textbf{Vanilla OGE-Aug} on $G$.

We will show that architectures of the form of Proposition 3.5 and specifically Vanilla OGE-Aug are incomplete on simple spectrum graphs, contradicting the claim in Proposition 3.5 such a representation is universal.

Consider the point clouds proposed by \citet{canonicalization-perspective}:
\begin{quote}

\begin{align}
U_1 = [u_{11}, u_{12}] = \begin{pmatrix}
1 & -1 & 1 & -1 \\
2 & 3 & 4 & 5
\end{pmatrix}^{\top}, \\
U_2 = [u_{21}, u_{22}] = \begin{pmatrix}
-1 & 1 & 1 & -1 \\
2 & 3 & 4 & 5
\end{pmatrix}^{\top}.
\end{align}

Suppose the first column eigenvector of $U_1$ and $U_2$ corresponds to eigenvalue $\lambda_1 = 1$, the second column eigenvector of $U_1$ and $U_2$ corresponds to eigenvalue $\lambda_2 = 2$, and other eigenvectors not shown corresponds to eigenvalue $0$ (so we safely ignore them). Then the Laplacian matrices corresponding to $U_1$ and $U_2$ are:

\begin{align}
L_1 &= \lambda_1 u_{11} u_{11}^{\top} + \lambda_2 u_{12} u_{12}^{\top} = \begin{pmatrix}
9 & 11 & 17 & 19 \\
11 & 19 & 23 & 31 \\
17 & 23 & 33 & 39 \\
19 & 31 & 39 & 51
\end{pmatrix}, \\
L_2 &= \lambda_1 u_{21} u_{21}^{\top} + \lambda_2 u_{22} u_{22}^{\top} = \begin{pmatrix}
9 & 11 & 15 & 21 \\
11 & 19 & 25 & 29 \\
15 & 25 & 33 & 39 \\
21 & 29 & 39 & 51
\end{pmatrix}.
\end{align}
\end{quote}

We will now demonstrate the model in Proposition 3.5 will be unable to distinguish $U_1$ and $U_@$, regardless of the choice of the GNN.

First, consider an arbitrary $O(1)-$invariant representation $f : \mathbb{R}^n \to \mathbb{R}^n$. We will show that $f(U_1)$ and $f(U_2)$ are identical.

By the permutation equivariance and $O(1)$ invariance:

\begin{equation}\label{eq:perm-equi}
    f(u_{11}) = f(-u_{11})=f(P_{11}u_{11}) = P_{11}f(u_{11})
\end{equation}

where $P_{11}$ is any permutation that satisfies $P_{11}u_{11}=-u_{11}$. Therefore $P_{11}$ can be chosen to be $\sigma_1 \triangleq \left( 1 \; 2\right)\left(3\; 4 \right)$ or $\sigma_2 \triangleq \left(1\; 4 \right)\left( 2 \; 3 \right).$

By Equation \ref{eq:perm-equi}, and since equality is a transitive relation, it holds that $f(u_{11})(i) =f(u_{11})(j) $ for any $i$ and $j$ in the same orbit under the group $<\sigma_1, \sigma_2>,$ the group generated by $\sigma_1$ and $\sigma_2.$ It is easy to check any pair $(i, j) \in \{ 1,2,3,4\}^{2}$ can be transposed under a group element in the generated group. Therefore, $f(u_{11})$ is a constant function. Analogous arguments yield $f(u_{21})$ is also constant.

Note that for $P_{12} \triangleq (1\; 2)(3\; 4),$ it holds that $$f(u_{11})\underbrace{=}_{f(u_{11}) \text{ is constant}} Pf(u_{11}) \underbrace{=}_{\text{perm. equivariance}} f(Pu_{11}) = f(u_{21})$$

Therefore, $f(u_{11})=f(u_{21}).$ Moreover, the second eigenvectors, $u_{12}$ and $u_{22}$ of $U_1$ and $U_2$, respectively, are identical therefore clearly $f(u_{12})=f(u_{22})$.

This analysis naturally extends to a proper eigendecomposition (orthonormal eigenvectors of a graph as proposed by \citet{canonicalization-perspective} in the proof of their Corollary 3.5 \cite{canonicalization-perspective}.

Therefore, as any universal, invariant set representation is the same on both $U_1$ and $U_2$, the input to the network will be identical per its definition, and thus for their corresponding graphs $G_1$ and $G_2$ and identical node features $X_{G_1}$ and $X_{G_2}$, respectively it holds that

$$r(G_1, X_{G_1}) = r(G_2, X_{G_2})$$

yet $G_1$ and $G_2$ are non-isomorphic, thus Vanilla OGE-Aug is incomplete.

\end{proof}
\subsection{Proof for equiEPNN strictly more expressive}

\corStronger*
\begin{proof}
We show that after a single iteration, the equivariant update step can yield new matrices $U^{(t)}, V^{(t)}, t=1$ which have no zeros.
We can choose the update function $\mathrm{UPDATE}_{(1,2)}$ such that $\mathrm{UPDATE}_{(1,2)}(v_5 \odot v_5, v_1 \odot v_1, v_5 \odot v_1 ) \triangleq (1,1,0,0,0,0, 0)$ and for all other values we define it as $\Vec{0}.$

After a single iteration $U^{(1)}$ and $V^{(1)}$ will be 

\setcounter{MaxMatrixCols}{20}
$$U^{(1)T}=
\begin{pmatrix}
z_0 & z_1 & z_2 & z_3 & z_0 & 0_2 &0_2 & 0_2 &  z_0 & z_1 & z_2 &z_3\\
z_0 & z_1 & z_2 & z_3 & z_0 & z_1 &z_2 & z_3 &  0_2 & 0_2 & 0_2 & 0_2\\
0_2 & 0_2 & 0_2 & 0_2 & z_0 & z_1 &z_3 & z_2 & z_0 & z_2 & z_1 & z_3

\end{pmatrix}
$$ 

$$V^{(1)T}=
\begin{pmatrix}
z_0 & z_1 & z_2 & z_3 & z_0  & 0_2 &0_2 & 0_2 &   z_0 & z_1 & z_2 &z_3\\
z_0 & z_1 & z_2 & z_3 &z_0  & z_1 &z_2 & z_3 &   0_2 & 0_2 & 0_2 & 0_2\\
0_2 & 0_2 & 0_2 & 0_2 &z_1  & z_0 &z_2 & z_3 &  z_2 & z_0 & z_3 & z_1
\end{pmatrix}
$$ 

Since there exists a column (the fifth column) such that all its entries are non-zero in both $U^{(1)T}$ and $V^{(1)T}$, from Theorem \ref{thm:2}, we know that a single iteration of EPNN, and hence also of equiEPNN, can separate $U^{(1)T}, V^{(1)T}$. In conclusion, two iterations of equiEPNN are sufficient for separation. 
\end{proof}
\section{Experiments}

\subsection{Dataset statistics}

We surveyed the graph spectra of popular datasets to verify the need for more expressive architectures based on graph properties. We now further specify the meaning of each row in Table \ref{tab:graph-stats} in Table \ref{tab:exp-tab-one}.

\begin{table}[htbp]\label{tab:exp-tab-one}
\centering
\begin{tabular}{|p{0.45\textwidth}|p{0.45\textwidth}|}
\hline
\multicolumn{2}{|c|}{\textbf{Eigenvalue Characteristics}} \\
\hline
Graphs with Distinct Eigenvalues & Graphs where all eigenvalues have multiplicity 1, meaning each eigenvalue appears exactly once in the spectrum \\
\hline
Graphs with Multiplicity 2 Eigenvalues & Graphs that have at least one eigenvalue that appears exactly twice in the spectrum \\
\hline
Graphs with Multiplicity 3 Eigenvalues & Graphs that have at least one eigenvalue that appears exactly three times in the spectrum \\
\hline
Avg. Number of Multiplicity 2 Eigenvalues & The average number of eigenvalues (across all analyzed graphs) that have multiplicity exactly 2 \\
\hline
Avg. Number of Multiplicity 3 Eigenvalues & The average number of eigenvalues (across all analyzed graphs) that have multiplicity exactly 3 \\
\hline
\multicolumn{2}{|c|}{\textbf{Eigenvector Properties}} \\
\hline
Average Ratio of Zeros & The average proportion of zero entries found in the eigenvectors across all analyzed graphs \\
\hline
Average Number of Zeros & The average count of zero entries in the eigenvectors across all analyzed graphs \\
\hline
Graphs with a Full Row & Graphs that have at least one eigenvector with no zero entries (i.e., a "full row" in the eigenvector matrix) \\
\hline
Graphs with $\leq 1$ Zero per Eigenvector & Graphs where each eigenvector has at most one zero entry \\
\hline
Graphs with Total Zeros $<$ Vertices & Graphs where the total number of zero entries across all eigenvectors is less than the number of vertices in the graph \\
\hline
Graphs Meeting Any Condition & Graphs that satisfy at least one of the specified eigenvector properties listed above \\
\hline
\end{tabular}
\caption{Explanation of Surveyed Graph Spectral Properties}
\label{tab:eigen_properties}
\end{table}
\subsection{MNIST Superpixel}

\begin{table}[H]
\centering
\caption{MNIST Superpixel Experiment Configuration}
\label{tab:mnist_experiment_config}
\begin{tabular}{lcc}
\toprule
\textbf{Parameter} & \textbf{Default Value} & \textbf{Description} \\
\midrule
k\_values & [3, 8, 16] & List of k values for positional encoding dimensions \\
epochs & 30 & Number of training epochs \\
batch\_size & 32 & Training batch size \\
data\_dir & 'data' & Data directory path \\
device & 'cuda' & Computing device (CUDA if available) \\
early\_stopping & 10 & Early stopping patience \\
output\_dir & 'results' & Output directory for results \\
coord\_update\_options & [True, False] & Coordinate update configurations \\
random\_seed & 42 & Random seed for reproducibility \\
\bottomrule
\end{tabular}
\end{table}

\begin{table}[H]
\centering
\caption{MNIST Superpixel Network Hyperparameters}
\label{tab:mnist_network_hyperparams}
\begin{tabular}{lcc}
\toprule
\textbf{Parameter} & \textbf{Default Value} & \textbf{Description} \\
\midrule
num\_features & 1 & Input node features (MNIST characteristic) \\
num\_classes & 10 & Output classes (MNIST digits 0-9) \\
hidden\_dim & 64 & Hidden layer dimension \\
num\_layers & 3 & Number of EGNN layers \\
pos\_enc\_dim & k & Positional encoding dimension (varies: 3, 8, 16) \\
dropout & 0.2 & Dropout rate \\
lr & 0.0005 & Learning rate \\
weight\_decay & 1e-5 & Weight decay for regularization \\
norm\_features & True & Normalize node features \\
norm\_coords & True & Normalize coordinates \\
coord\_weights\_clamp & 1.0 & Clamping value for coordinate weights \\
with\_pos\_enc & True & Use positional encoding \\
with\_proj & False & Use edge projectors \\
with\_virtual\_node & False & Use virtual node \\
update\_coords & True/False & Coordinate update flag (both tested) \\
\bottomrule
\end{tabular}
\end{table}

\begin{table}[H]\label{tab:mnist-train}
\centering
\caption{MNIST Superpixel Training Configuration}
\label{tab:mnist_training_config}
\begin{tabular}{lcc}
\toprule
\textbf{Parameter} & \textbf{Value} & \textbf{Description} \\
\midrule
Optimizer & Adam & Optimization algorithm \\
Loss Function & NLL Loss & Negative log-likelihood loss \\
Scheduler & ReduceLROnPlateau & Learning rate scheduler \\
LR Reduction Factor & 0.5 & Factor for LR reduction \\
LR Patience & 5 & Scheduler patience \\
Min LR & 1e-6 & Minimum learning rate \\
Gradient Clipping & 1.0 & Maximum gradient norm \\
Early Stopping Patience & 10 & Training patience \\
\bottomrule
\end{tabular}
\end{table}
In our first experiment, we applied the proposed method on a classical task of handwritten digit classification in the MNIST dataset~\cite{lecun1998mnist}. While almost trivial by today's standards, we use this example to verify the theoretical claims regarding expressivity on simple spectrum graphs. Our experimental setup employed both EPNN (coordinate updates disabled) and equiEPNN (coordinate updates enabled) as our models exclusively on the superpixel-based graph representation from the MNISTSuperpixels dataset. In this approach, each 28 $\times$ 28 image was converted into a graph where vertices correspond to superpixels and edges represent their spatial adjacency relations, each image was represented as a different graph. We tested our models with different positional encoding dimensions of $k = 3, 8, 16$ to evaluate performance across varying levels of spectral information.

For details configutions see Tanbes \ref{tab:mnist_experiment_config}, \ref{tab:mnist_network_hyperparams}, and \ref{tab:mnist_training_config}.

\subsection{Realiable Expressivity}

The BREC \cite{wang2023empirical} dataset is a graph expressivity benchmark consisting of highly symmetric graphs that high-order GNNs struggle at distinguishing, which was used by \cite{zhang2024expressive} to check the expressivity of EPNN. We implemented EPNN and equiEPNN via the popular EGNN \cite{satorras2021n} framework and obtained statistically identical results shown in Table \ref{tab:epnn_brec}.

\begin{table}[H] 

\centering
\caption{Empirical performance of different GNNs on BREC (in percentages.) (Using k=3 spectral features, results  of non-EPNN models from \cite{zhang2024expressive})}
\label{tab:epnn_brec}
\begin{tabular}{ccccccc}
\hline
Model & WL class & Basic & Reg & Ext & CFI & Total \\
\hline
Graphormer & SPD-WL & 26.7 & 10.0 & 41.0 & 10.0 & 19.8 \\
NGNN & SWL & 98.3 & 34.3 & 59.0 & 0 & 41.5 \\
ESAN & GSWL & 96.7 & 34.3 & 100.0 & 15.0 & 55.2 \\
PPGN & 3-WL & 100.0 & 35.7 & 100.0 & 23.0 & 58.2 \\
\hline
EPNN & EPWL & 100.0 & 35.7 & 100.0 & 4.0 & 53.5 \\

Equi-EPNN & N/A  & 100.0 & 35.7 & 100.0 & 4.0 &  53.5 \\
\hline
\end{tabular}
\end{table}

\subsection{Eigenvector Canonicalization}

We specify the problem setup for eigenvector canonicalization and our proposed method.

\begin{definition}[Eigenvecor Canonicalization]
    A \textit{canonicalization} of an eigenvector $v\in \RR^{n}$ is a map $
        \phi : \RR^{n} \to \RR^{n}$ such that for every $s\in O(1) \simeq \{-1,1\},$ it holds that $\phi(sv)= \phi(v)$ and is permutation equivariant, that is for every permutation $\sigma$, $\phi(\sigma v)=\sigma\phi(v).$

\end{definition}

We now define the following eigenvector canonicalization map via the steps

\begin{enumerate}
    \item For given eigenvectors $V\in \RR^{n\times k}$ corresponding to distinct eigenvalues, we run equiEPNN for $T$ iterations, to obtain the equivariant output $V^{(T)}$
    \item We sum over the columns to obtain a matrix $S=\mathrm{diag}(s_1,s_2,\ldots,s_k)$ where $s_i \triangleq \mathrm{sign}(\sum_{j=1}^{n}V^{(T)}(i,j)) \in \{-1,+1\}$.
    \item Canonicalize the eigenvectors via $SV.$
\end{enumerate}

This defines an eigenvector canonicalization map $\psi : \RR^{n \times k} \to \RR^{n \times k}$ where $\psi(V) = SV$ for the $S(V)$ defined above. This map is naturally permutation equivariant, and it is easy to check that it is sign invariant. 

As this maps canonicalized the original eigenvectors via aggregating global graph information that depends on the entire graph eigendecomposition and not each eigenvector separately, we obtain a map that practically achieves perfect canonicalization on ZINC \cite{Irwin2012}.

See Tables \ref{tab:network_hyperparams} and \ref{tab:experiment_config} for experiment configurations.
\begin{table}[H]
\centering
\caption{Eigenvector Canonicalization Configuration}
\label{tab:experiment_config}
\begin{tabular}{lcc}
\toprule
\textbf{Parameter} & \textbf{Default Value} & \textbf{Description} \\
\midrule
subset\_size & 100 & Number of ZINC graphs to test \\
k\_projectors & 10 & Number of top eigenvalue projectors to use \\
num\_workers & 4 & Number of workers for data loading \\
device & CUDA/CPU & Computing device (CUDA if available) \\
precision & float64 & Default tensor precision \\
\bottomrule
\end{tabular}
\end{table}

\begin{table}[H]
\centering
\caption{Network Hyperparameters for Eigenvector canonicalization}
\label{tab:network_hyperparams}
\begin{tabular}{lcc}
\toprule
\textbf{Parameter} & \textbf{Default Value} & \textbf{Description} \\
\midrule
num\_layers & 5 & Number of message passing layers \\
emb\_dim & 128 & Embedding dimension \\
in\_dim & 128 & Input feature dimension \\
proj\_dim & 10 & Projection dimension \\
coords\_weight & 3.0 & Coordinate update weight \\
activation & relu & Activation function \\
norm & layer & Normalization type \\
aggr & sum & Aggregation function \\
residual & False & Use residual connections \\
edge\_attr\_dim & 20 & Edge feature dimension (2 × k\_projectors) \\
\bottomrule
\end{tabular}
\end{table}

\end{document}